\documentclass{article}

        \usepackage[nonatbib,final]{neurips_2022}

\usepackage[utf8]{inputenc} 
\usepackage[T1]{fontenc}    
\usepackage{hyperref}       
\usepackage{url}            
\usepackage{booktabs}       
\usepackage{amsfonts}       
\usepackage{nicefrac}       
\usepackage{microtype}      
\usepackage{xcolor}         
\usepackage{algorithm}
\usepackage{algorithmic}

\usepackage{microtype}
\usepackage{graphicx}
\usepackage{subfigure}
\usepackage{booktabs} 

\usepackage{hyperref}

\usepackage{amsmath}
\usepackage{amssymb}
\usepackage{mathtools}
\usepackage{amsthm}
\usepackage{wasysym}

\usepackage[capitalize,noabbrev]{cleveref}

\theoremstyle{plain}
\newtheorem{theorem}{Theorem}[section]

\newtheorem{lemma}[theorem]{Lemma}
\newtheorem{corollary}[theorem]{Corollary}
\theoremstyle{definition}

\theoremstyle{remark}

\usepackage[textsize=tiny]{todonotes}

\title{Connecting adversarial attacks and optimal transport for domain adaptation}

\author{%
  Arip Asadulaev \\
  ITMO University\\
  Artificial Intelligence Research Institute\\
  \texttt{aripasadulaev@itmo.ru} \\
    \And
    Vitaly Shutov \\
  ITMO University\\
  \texttt{vitaly.shutov1@gmail.com} \\
   \And
    Alexander Korotin\\
  Skolkovo Institute of Science and Technology\\
  Artificial Intelligence Research Institute\\
  \texttt{a.korotin@skoltech.ru} \\
    \And
    Alexander Panfilov \\
  Eberhard Karls Universität Tübingen\\
  \texttt{kotekjedi@gmail.com} \\
     \And
    Andrey Filchenkov \\
  ITMO University\\
  \texttt{afilchenkov@itmo.ru} \\
}
\begin{document}
\maketitle

\everypar{\looseness=-1}

\begin{abstract}
We present a novel algorithm for domain adaptation using optimal transport. In domain adaptation, the goal is to adapt a classifier trained on the source domain samples to the target domain. In our method, we use optimal transport to map target samples to the domain named \textit{source fiction}. This domain differs from the source but is accurately classified by the source domain classifier. Our main idea is to generate a \textit{source fiction} by $c$-cyclically monotone transformation over the target domain. If samples with the same labels in two domains are $c$-cyclically monotone, the optimal transport map between these domains preserves the class-wise structure, which is the main goal of domain adaptation. To generate a \textit{source fiction} domain, we propose an algorithm that is based on our finding that adversarial attacks are a $c$-cyclically monotone transformation of the dataset. We conduct experiments on Digits and Modern Office-31 datasets and achieve improvement in performance for simple discrete optimal transport solvers for all adaptation tasks.
\end{abstract}

\section{Introduction}\label{sec-intro}
Optimal Transport (OT) is a simple framework for solving mass-moving problems for probability distributions. Optimal transport was successfully applied in mathematics~\cite{ferradans2014regularized}, economics~\cite{reich2013nonparametric}, and machine learning~\cite{arjovsky2017wasserstein, mroueh2019wasserstein, solomon2015convolutional, colombo2021automatic}, especially in domain adaptation (DA) problem~\cite{CourtyFTR15, PerrotCFH16, rakotomamonjy2020optimal}.  The discrete OT algorithms for DA are used to map the labeled source samples to the unlabeled or partially labeled samples in the target domain. 

The optimal transport results in domain adaptation depend on the cost function $c$~(\wasyparagraph\ref{sec-prelim}). The optimal transport aims to find a $c$-cyclically monotone~\cite{mccann1995existence} map; it is impossible to perturb such a map and get a more cost-efficient one~\cite{villani2008optimal}. However, a cyclically monotone map with Euclidean cost may incorrectly capture the class-wise structure between domains, leading to inaccuracy of domain adaptation by optimal transport. 

To correctly capture the class-wise structure during domain adaptation via optimal transport, we present a novel algorithm that maps the target to the domain named \textit{source fiction}. The \textit{source fiction} domain differs from the real source but it is also \textbf{classified correctly} by the source classifier. At the same time, mapping the target to the \textit{source fiction} (by OT with Euclidean cost) \textbf{saves the class-wise structure}. 

To generate the \textit{source fiction}, we propose to use the $c$-cyclically monotone transformation of the target domain. The idea is that if we transform the target by $c$-cyclically monotone map, then the corresponding sample in the \textit{source fiction} will have the same label. More details about $c$-cyclical monotonicity see in Section (\wasyparagraph\ref{sec-monotone})

As the $c$-cyclically monotone transformation that makes target samples classified correctly by the source classifier, we propose to use an adversarial attack~\cite{DBLP:journals/corr/SzegedyZSBEGF13}. We demonstrate that Iterative Fast Sign Gradient Descent (FSGD)~\cite{goodfellow2014explaining} is a $c$-cyclically monotone transformation over the dataset with quadratic cost. In other words, we replace the source domain with a new domain constructed by the gradients of the source classifier over the target samples. It is important to note that in our method, the adversarial attack turns samples into the correctly classified instead of "fooling" classifier, hence we call it \textit{inverse adversarial attack}. It is necessary to have some labeled samples in the target domain to apply an inverse adversarial attack. Therefore, our method is semi-supervised. 

In our experiments, we show that in semi-supervised settings, our method improves the accuracy of discrete OT solvers on a range of DA tasks (\wasyparagraph \ref{sec-results}). The improvement of discrete OT solvers is important, as this method has low complexity and solves DA faster than the deep neural network-based DA algorithms~\cite{GaninL15, cdan, mmd, DAN, JAN}. 

\textbf{Contribution:} We propose a novel method for semi-supervised domain adaptation with optimal transport. We prove that adversarial attacks are a $c$-cyclically monotone transformation of the dataset (\wasyparagraph \ref{sec-cmonotonicity}). Using this property, we propose a novel algorithm that improves the performance of the OT solvers on the range of DA problems (\wasyparagraph \ref{sec-experiments}).

\section{Background}
\label{sec-prelim}
\subsection{Optimal transport}

OT can be a simple solution for the domain adaptation problem. 
OT aims at finding a cost-effective mapping $T : {X} \rightarrow {Y}$ of the two probability measures $\mathbb{P}=\sum_{i=1}^{n} {a}_{i} x_{i} \text { and } \mathbb{Q}=\sum_{j=1}^{m} {b}_{j} y_{j}$ with respect to the cost function $c: {X} \times {Y} \rightarrow\mathbb{R}_{+}$, where $a_i$ and $b_i$ are the values of the Dirac function at $x_i$ and $y_j$ correspondingly. Monge's problem was the first example of the OT problem~\cite[\wasyparagraph 3] {villani2008optimal} and can be formally expressed as follows:
\begin{equation}
    \inf _{T \# \mathbb{P}=\mathbb{Q}} \int_{\Omega_{\mathbb{P}}} c(\mathbf{x}, T(\mathbf{x})) \mathbb{P}(\mathbf{x}) d \mathbf{x}
    \label{monge}
\end{equation}

Monge's formulation of OT aims at finding a map ${T}$ where ${T} \# \mathbb{P}_s=\mathbb{Q}_t$ represents the mass preserving push forward operator.
In Monge's formulation, for two given measures $\mathbb{P}$ and $\mathbb{Q}$, the existence of a transport map ${T}$ is not only non-trivial but it also may not exist~\cite[\wasyparagraph 5.1] {villani2008optimal}.

Kantorovich proposed the relaxation of the Monge's problem\ref{monge} and presented the formulation in which a solution always exists~\cite[\wasyparagraph 5.1] {villani2008optimal}. 
The Kantorovich problem aims to find a joint distribution over the $\mathbb{P}$ and the $\mathbb{Q}$ that determines how the mass is allocated.
To find an optimal solution, it is necessary to build the cost matrix for all $\mathbf{x} \in {X}$ and $\mathbf{y} \in {Y}$ samples:

\begin{equation}
    M_{{X}{Y}} \stackrel{\text { def }}{=}\left[c\left(\mathbf{x}_{i}, \mathbf{y}_{j}\right)^{p}\right]_{i j} 
\end{equation}   

Having the cost matrix $M_{{X} {Y}}$, we are searching for the optimal coupling $\mathbf{\gamma}$, that minimizes the displacement cost between two probability measures $\mathbb{P}$ and $\mathbb{Q}$  
\begin{equation}
    W_{p}^{p}(\mathbb{P}, \mathbb{Q})=\min _{{\gamma} \in U(a, b)}\left\langle {\gamma}, M_{{X} {Y}}\right\rangle
    \label{primalKant}
\end{equation}
with the constraints to the coupling ${\gamma} \in U({a}, {b})$ such that:
\begin{equation}
  U({a}, {b}) \stackrel{\text { def }}{=}\left\{\mathbf{\gamma} \in \mathbb{R}_{+}^{n \times m} \mid \mathbf{\gamma} \mathbf{1}_{m}={a}, \mathbf{\gamma}^\top \mathbf{1}_{n}={b}\right\}
\end{equation}

The infimum of this optimization problem induces the Wasserstein distance, and coupling $\gamma$ gives us a non-bijective map between probability measures $\mathbb{P}$ and $\mathbb{Q}$. 

\subsubsection{Cyclical monotonicity}
\label{sec-monotone}
The main geometric property of the OT maps is $c$-cyclical monotonicity. 
Formally the map is $c$-cyclically monotone if for all points $\mathbf{x}_0\dots \mathbf{x}_i$, $\mathbf{y}_0\dots \mathbf{y}_i$ holds: 
\begin{equation}
    \sum_{i=1}^{N} c\left(\mathbf{x}_{i}, \mathbf{y}_{i}\right) \leq \sum_{i=1}^{N} c\left(\mathbf{x}_{i}, \mathbf{y}_{i+1}\right)
\end{equation}
A $c$-cyclically monotone map cannot be improved in terms of the cost function $c$~\cite[\wasyparagraph 5] {villani2008optimal}. It is critical to have this property between two domains in DA because OT solver always \textbf{aims to find a $c$-cyclically monotone map}~\cite[\wasyparagraph 5] {villani2008optimal}. If two distribution samples $\mathbf{x}_n$ and $\mathbf{y}_n$ are $c$-cyclically monotone, and at the same time, their labels are equal for each $n$, OT can map one distribution to another while preserving a class-wise structure. In practice, when cost $c$ is Euclidean, samples in different domains with the same labels are not always the closest. This leads to the behavior when OT maps samples into the wrong classes. 

\section{Related Work}

\label{sec-related}
\textbf{Optimal transport for domain adaptation:} Optimal transport map labeled source samples in $\mathbb{Q}$ to the unlabeled or partly labeled samples in target $\mathbb{P}$. Different linear programming solvers can be used to find a solution to the mass transport problem; for example, the Dantzig simplex~\cite{nash2000dantzig} methods are actively used in various OT solvers like Earth Mover’s Distance~(EMD)~\cite{CourtyFTR15,flamary2021pot}. For differentiable OT, the Sinkhorn algorithm~\cite{Cuturi13} was proposed. The Sinkhorn is based on the matrix-vector multiplication operations and can be combined with various regularizations like group lasso regularization~(L1L2) and Laplacian regularization~(L1LP)~\cite{CourtyFTR15}. To make the OT applicable to the out-of-sample mapping, a linear OT mapping estimator (OTLin) was proposed~\cite{PerrotCFH16}. OTLin jointly computes the coupling $\mathbf{\gamma}$~\eqref{primalKant} and maps $T$ linked to the original Monge problem~\eqref{monge}. 

The Inverse OT algorithms can reconstruct the cost function that saves the underlying data structure during mapping~\cite{li2019learning}. Recently it was shown that the cost function could be approximated by the neural network and the given labels~\cite{liu2019learning}. To train the cost function, it is necessary to solve the transport problem using the Sinkhorn algorithm~\cite{Cuturi13} at every optimization step~\cite{liu2019learning}. These methods are hardly scalable and have not been applied to solve DA problems. 

The connection between OT and deep networks was proposed for the unsupervised DA~\cite{damodaran2018deepjdot} and transfer learning~\cite{li2020representation}. In DA with label and target shift problems, OT methods align probability distributions between a few domains~\cite{redko2019optimal, rakotomamonjy2020optimal}. 

\textbf{Deep domain adaptation:} Alternately, the task of DA is to make the target domain $\mathbb{P}$ samples closer to the source domain $\mathbb{Q}$ samples. In this setting, the goal is to adapt the source classifier $f_{\theta}$ to the unlabeled target domain samples $\mathbb{P}$~\cite{DavidBCKPV10, DavidLLP10, GermainHLM13}. Usually, a small number of labeled samples are available in the target domain. Standard approaches to this problem are distance-based algorithms ~\cite{mmd, DAN, JAN} or adversarial-based algorithms like DANN~\cite{GaninL15}, CDAN, CDAN-E~\cite{cdan}. These methods show high accuracy but require time-consuming computations. In comparison to the deep DA methods, OT has lower computational complexity and theoretical guarantees in domain adaptation~\cite{redko2017theoretical}. In our paper \textbf{we consider deep DA problem setting} to improve the accuracy of the source classifier $f_{\theta}$ on the target data using simple OT solvers instead of neural networks. 

\textbf{Adversarial attacks:} Previously, the various properties of adversarial examples were studied~\cite{Transfer, PapernotMG16, DBLP:journals/corr/abs-1905-02175}. Applications of adversarial examples for model accuracy improvements were also proposed~\cite{advprop, YangXLQSLL20}. The connection between OT and adversarial examples was studied in the context of robustness problems~\cite{pydi2020adversarial, bouniot2021optimal}. It has also been shown that the Sinkhorn algorithm could be used to find adversarial perturbations with respect to Wasserstein ball~\cite{DBLP:conf/icml/WongSK19}.  To the best of our knowledge, we propose the first method that connects the adversarial attacks and optimal transport for domain adaptation.

\section{Proposed method}
In this section, we present a new algorithm that generates a new domain by the inverse adversarial attack. Before introducing the algorithm, we outline motivation to use an adversarial attack to generate a \textit{source fiction} domain.

\subsection{Adversarial attacks are $c$-cyclically monotone}
\label{sec-cmonotonicity}

Adversarial examples are samples that are similar to the true samples but ``fool'' a classifier and tend to make incorrect predictions~\cite{DBLP:journals/corr/SzegedyZSBEGF13}. The phenomenon of the vulnerability of machine learning models to adversarial examples raises a great deal of concern in many learning scenarios~\cite{DBLP:conf/ccs/PapernotMGJCS17, YuanHZL19, DBLP:conf/iclr/SchottRBB19, xie2017adversarial}. With small changes described in~\ref{sec-algorithm}, the adversarial attack can turn any sample into an accurately classified one in the same way that is used to ``fool`` the classifier. In our paper, we consider the simplest adversarial attack: Fast Sign Gradient Descent (FSGD)~\cite{goodfellow2014explaining}. The iterative version of FSGD can be presented as:

\begin{equation}
    {\mathbf{x}}_{0}'={\mathbf{x}}, \quad {\mathbf{x}'}_{i+1}=\operatorname{clip}_{\mathbf{x},\varepsilon}\left\{\mathbf{x'}_{i}-\alpha \operatorname{sign}\left(\nabla_{x} L\left(\theta, {\mathbf{x}'}_{i}, \mathbf{y}\right)\right)\right\}
    \label{fsgd}
\end{equation}
With sample $\mathbf{x}$, label $\mathbf{y}$, and classifier $f_\theta$, we can obtain adversarial examples using gradient descent perturbations, maximizing the loss $L$ on a sample $\mathbf{x}$ with respect to some perturbation size $\varepsilon$. Adversarial attacks with a small parameter $\varepsilon$ are $c$-cyclically monotone maps, i.e., OT. With mild assumptions on adversarial attack we prove that it is a $c$-cyclically monotone transformation of the data w.r.t. the quadratic cost $c(\mathbf{x},\mathbf{y})=\frac{1}{2}\|\mathbf{x}-\mathbf{y}\|^{2}$. 

\begin{lemma}[cyclical monotonicity of small perturbations of a dataset.]
Let $\mathbf{x}_{1},\dots,\mathbf{x}_{N}\in\mathbb{R}^{D}$ be a dataset of $N$ distinct samples. 
Let $\mathbf{x}_{1}',\dots,\mathbf{x}_{N}'$ be its $\leq \varepsilon$-perturbation, i.e. $\|\mathbf{x}_n-\mathbf{x}_n'\|\leq \varepsilon$ for all $n=1,2,\dots,N$. Assume that $\varepsilon\leq \frac{1}{2}\min\limits_{n_1,n_2}\|\mathbf{x}_{n_1}-\mathbf{x}_{n_2}\|$, i.e. the perturbation does not exceed $\frac{1}{2}$ of the minimal pairwise distance between samples, then $\leq \varepsilon$-perturbation is $c$-cyclically monotone.
\end{lemma}
Adversarial attacks are small $\varepsilon$ perturbations of the dataset samples, we immediately obtain:
\begin{corollary}
\label{corollary-main}
Let $x_{1},\dots,x_{N}\in\mathbb{R}^{D}$ be a dataset of $N$ distinct samples. Then any adversarial attack $\mathbf{x}_{n}\mapsto \mathbf{x}_{n}'$ on the dataset with $\varepsilon\leq \frac{1}{2}\min\limits_{n_1,n_2}\|\mathbf{x}_{n_1}-\mathbf{x}_{n_2}\|$ is $c$-cyclically monotone.
\end{corollary}

The proof of lemma is in Appendix~(\wasyparagraph\ref{lemma-main}). Corollary \ref{corollary-main} suggests to use the \textit{optimal} map to transform target domain $\mathbb{Q}$ to domain $\mathbb{P}'$ formed by the $c$-cyclically monotone \textit{inverse adversarial attack}.

\subsection{Domain adaptation with \textit{source fiction}}
\label{sec-algorithm}

\begin{figure*}[h!]
 \centering
 \includegraphics[width=\textwidth]{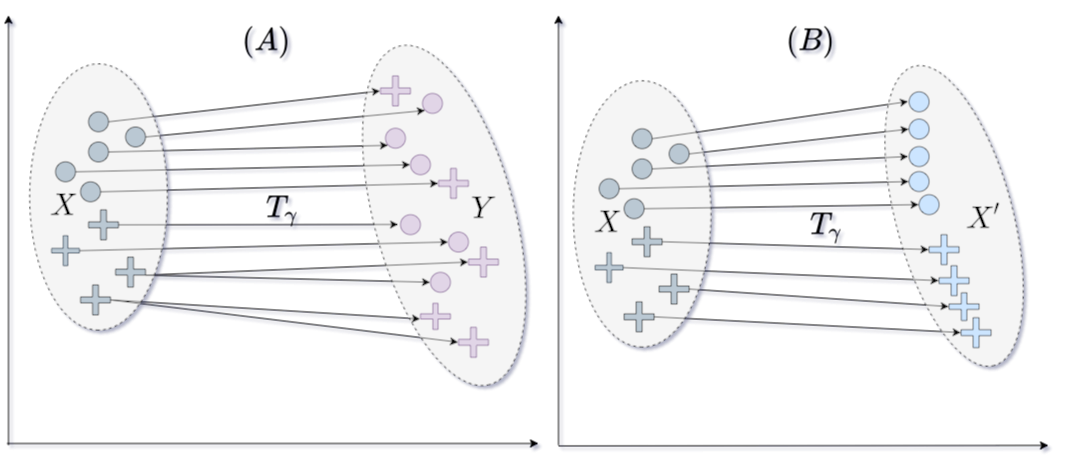}
 \label{fig:cmot}
 \caption{Illustration of the difference between two settings. (A) Optimal transport $T_\gamma$ between the target $X$ and source $Y$ domain samples. (B) Optimal transport mapping between target and the proposed \textit{source fiction} domain samples $X'$. The figure illustrates the example with the two classes in each domain (circles and crosses). By the lines, connected the "closest w.r.t the Euclidean distance" samples between the two domains. In (A), the samples from the same class in the two domains are not always closest. In (B), samples in the target domain are the closest to the corresponding class sample in the \textit{source fiction}.}
\end{figure*} 

\begin{algorithm}[H]
   \label{alg:algo}
\begin{algorithmic}
\STATE {\bfseries Input:} Classifier $f_{\theta}$, OT algorithm $T$,
 source samples ${Y}$, labeled ${X}_l$ and all target samples ${X}$, perturbations size $\varepsilon$
\STATE {\bfseries Initialize:} $\varepsilon \leq \frac{1}{2}\min\limits_{n_1,n_2}\|\mathbf{x}_{n_1}-\mathbf{x}_{n_2}\|$ for $\mathbf{x}_n$ in labeled $X_{l}$
\STATE Pretrain classifier $f_{\theta}$ on source
\STATE ${Y'}\leftarrow \emptyset$
    \FOR{ $\mathbf{x}, \mathbf{y} \in X_{l}$}
        \STATE $\mathbf{x'}\leftarrow\mathbf{x}$ 
        \FOR{some iterations}
            \STATE $\mathbf{x'}\leftarrow\operatorname{clip}_{\mathbf{x}, \varepsilon}\left\{\mathbf{x}_{n}'-\alpha \operatorname{sign}\left(\nabla_{x} L\left(\theta, \mathbf{x'}, \mathbf{y}\right)\right)\right\}$
        \ENDFOR
        \STATE Add $\mathbf{x}'$ to ${X'}$ 
\ENDFOR
\STATE Find a map ${X} \rightarrow {X'}$ using $T$.
\STATE Apply the classifier $f_{\theta}$ to the output of $T$.
\RETURN Trained classifier $f_{\theta}$ and map $T$.
\end{algorithmic}
\caption{Algorithm for DA with \textit{source fiction}}
\end{algorithm}

As we proved, adversarial attacks are $c$-cyclically monotone transformation over the dataset~\ref{corollary-main}.To manage this property, we replace the source with the distribution that is cyclically monotone to the target distribution and is, at the same time, accurately classified by a source classifier. 

In our method, we use a targeted FSGD adversarial attack (Eq.~\eqref{fsgd}) to source classifier $f_{\theta}$, with target label $\mathbf{y}$ equal to the true class of the given sample. Such attack adds to the image features of the class it really belongs to~\cite{DBLP:journals/corr/abs-1905-02175}. By inverse adversarial attack we obtain a new distribution $\mathbb{P}'$ with samples ${X}\subset \mathbb{R}^{D}$. Following corollary~\ref{corollary-main} to obtain monotonicity, we set the size of perturbation $\varepsilon \leq \frac{1}{2}\min\limits_{n_1,n_2}\|\mathbf{x}_{n_1}-\mathbf{x}_{n_2}\|$ for all $\mathbf{x}_n$ in ${X}_{l}$.

For each labeled sample $X_l$ in the target distribution, we obtain a corresponding samples ${X}'$. While ${X}'$ is a cyclically monotone transformation of the target, we have a low quadratic cost between each target sample ${X_l}$ and its corresponding sample in ${X}'$. To apply adaptation of classifier $f_{\theta}$ on the target distribution samples without labels, we use OT to find a map between all samples from the distribution ${X}$ and ${X}'$. The pipeline of OT for DA with samples ${X}'$ is presented in Algorithm~\ref{alg:algo}.

\section{Experiments}
\label{sec-experiments}

In this section, we test our method on two types of datasets (\wasyparagraph \ref{sec-datasets}). The experiments aim to demonstrate that our method improves the performance of fundamental discrete OT algorithms. Besides, we compare different deep DA baselines \ref{sec-baselines}. Discussion on empirical complexity is presented in Section (\wasyparagraph \ref{sec-complexity}). Additionally, we conduct an ablation study on the $\varepsilon$ parameter to show the stability of our method in different settings of the inverse adversarial attack. The code is written in \textit{PyTorch} framework and will be made public. We give more training details (pre-processing, architectures, etc.) in Appendix (\wasyparagraph\ref{sec-details}). 

\subsection{Datasets}
\label{sec-datasets}
\textbf{Digits:} We evaluated our method on Digits datasets MNIST~\cite{mnisthandwrittendigit}, USPS~\cite{uspsdataset}, SVNH~\cite{netzer2011reading}, and MNIST-M~\cite{GaninL15}. Each dataset consists of 10 classes of digit images with different numbers of train and test samples. We resized the images to the ($32\times 32$) pixel size as a pre-processing step.

\textbf{Modern Office-31:} Besides the Digits dataset that consists of only ten classes in each domain, we tested our method on the Modern Office-31 dataset~\cite{ringwald2021adaptiope} with 31 classes per domain. The Modern Office-31 dataset is one of the most extensive and diverse datasets for DA. The dataset consists of three domains: Amazon (A), Synthetic (S), and Webcam (W). In comparison to the Digits and original Office-31 dataset~\cite{saenko2010adapting}, this dataset includes synthetic $\rightarrow$ real task, which is problematic. Additionally, we included the DLSR (D) domain from the original Office-31 to estimate our algorithm properly, see the results in Appendix (\wasyparagraph \ref{sec-details}). 

\subsection{Baselines}
\label{sec-baselines}
\textbf{Deep domain adaptation algorithms:} We considered the range of fundamental gradient-based DA techniques. First of all, we compared our method to the prominent adversarial-based approach DANN~\cite{GaninL15}, CDAN, CDAN-E~\cite{cdan}. We also tested the Maximum Mean Discrepancy (MMD)~\cite{mmd} based DA techniques like DAN~\cite{DAN} and JAN~\cite{JAN} . Additionally, we considered the Wasserstein distance-based method WDGRL~\cite{shen2018wasserstein}. We used implementation for this methods proposed in in ADA framework~\cite{adalib2020}. For the Modern Office-31 experiments we compared our method to the RSDA-DANN~\cite{ringwald2021adaptiope}, SymNet~\cite{zhang2019domain}, and CAN~\cite{kang2019contrastive} methods.

\textbf{Neural optimal transport:} Following the benchmark results of the neural OT algorithms benchmark~\cite{korotin2021neural}, we choose two methods to apply DA: W2GN~\cite{KorotinW2} and MM:R~\cite{nhan2019threeplayer, korotin2021neural}.  We used Dense ICNN~\cite{KorotinW2} with three hidden layers $[64, 64, 32]$ as potentials $\phi$ and $\psi$ in for W2GN and MM:R neural OT methods. Potentials are pertained to apply invariant transformation using Adam~\cite{Ruder16} optimizer with $lr$ equal to 1e-4. Methods are trained 300 epochs with Adam optimizer and $lr$ equal to $1e-3$.

\textbf{Discrete OT solvers:} We tested  several OT solvers in semi-supervised domain adaptation settings: EMD~\cite{CourtyFTR15}, Sinkhorn~\cite{Cuturi13}, SinkhornL1L2, SinkhornLPL2, and OTLin~\cite{PerrotCFH16}. Most of these algorithms are presented in the POT framework~\cite{flamary2021pot}, which provides state-of-the-art OT solvers for DA. For experiments, we used quadratic cost $c(\mathbf{x},\mathbf{y})=\frac{1}{2}\|\mathbf{x}-\mathbf{y}\|^{2}$ for each algorithm. Regularization size equals 4.0 for Sinkhorn, SinkhornL1L2, SinkhornLPL2, and OTLin; all other hyperparameters are equal to the default, presented in POT. 

\subsection{Settings}
\label{sec-settings}
For the source domain classifier, we trained ResNet50~\cite{he2016deep} to achieve ~$90+$ accuracy on the test set of each domain in Digits and Modern Office-31. The classifier was trained using Adam~\cite{Ruder16} optimizer with a $1e-3$ learning rate. The size of latent space before the output layer was equal to $2048$. After training, we applied DA by moving mass in the latent space of the source classifier. For discrete OT baselines, the available labels were used to penalty the transport plan by building a cost matrix $M$ with $M(i, j)=0$ when $\mathbf{x}_{i}$ and $\mathbf{y}_{j}$ labels are equal and $+\infty$ if not~\cite{courty2016optimal, yan2018semi}

\begin{table}[h!]\centering
\begin{center}
\begin{small}
\begin{tabular}{l|l|l l l l}
\hline
Class & Method & M $\rightarrow$ S    & S $\rightarrow$ M   & M $\rightarrow$ U      & M $\rightarrow$ MM \\ 
\hline
Deep DA &Source    &22.0	&79.0	 &74.1		&33.5 \\
     &DANN      &19.5    &61.7   &93.8      &37.5     \\
     &C-DAN     &11.5    &79.0   &90.7      &68.4     \\
     &CDAN-E    &11.3    &77.9   &90.3      &69.6     \\
     &DAN       &16.7    &54.8   &\textbf{95.0}      &47.0     \\
     &JAN       &11.5    &57.9   &89.5      &52.9     \\
     &WDGRL     &13.8    &59.5   &85.7      &52.0     \\
     &W2GN      &20.4	 &79.9   &89.1   	&\textbf{74.1}  \\  
     &MM:R      &20.3    &80.2   &78.0      &63.8     \\

\hline
Discrete &EMD	        &21.2	&68.7	&79.2		&56.1 \\
OT      &OTLin	        &21.8	&69.9	&84.1		&62.3\\ 
        &Sinkhorn	    &21.8	&68.8	&82.1		&55.7\\
        &SinkhornLpL1	&21.8	&68.8	&84.8		&55.7\\
        &SinkhornL1L2	&21.8	&68.8	&84.8		&55.7\\
\hline
Discrete &EMD	&{23.0}	&{86.3}	&{83.1}	&{62.7}\\  
OT (\textbf{Ours})&OTLin &\underline{\textbf{25.5}}	&\underline{\textbf{88.4}}	&\underline{89.3}		&\underline{64.5} \\
        &Sinkhorn  	    &{25.5}	&{86.2}	&{83.8}		&{62.9}\\
        &SinkhornLpL1 	&{25.5}	&{86.3}	&{88.3}		&{63.0} \\
        &SinkhornL1L2 	&{25.5}	&{86.3}	&{88.3}		&{63.0} \\ 
\hline 

\end{tabular}
\end{small}
\end{center}
\caption{Accuracy$\uparrow$ of different types of DA algorithms in the latent space of source classifier on Digits datasets. For discrete OT methods semi-supervised settings with 10 labeled samples in class are demonstrated. By \textbf{bold} we denote highest accuracy, by \underline{underline} we denote highest accuracy for an discrete OT algorithms}
\label{tab:digits50}
\end{table}

To create a source fiction, we used 50 steps FSGD with $\varepsilon$ equal to $0.45$. We found that this value allows to achieve strong perturbations and, at the same time, satisfies the proposed bound on all domains. The results are presented in Tables \ref{tab:digits50}, \ref{tab:office}. All values in the tables are averaged over the ten runs with randomly chosen sets of labeled samples in the target domain. The top part of the table represents deep DA and standard discrete OT algorithm results, the bottom part of the table presents the results of discrete OT with \textit{source fiction}.

\subsection{Results}
\label{sec-results}
Our method demonstrates improvement for all adaptation tasks. The simplest EMD method is less accurate than other methods, and OTLin accuracy is slightly higher for all domains. The Sinkhorn algorithm with group lasso and Laplacian regularizations did not provide notable improvements over standard Sinkhorn. The results with only three known labels in class are presented in Appendix~\ref{tab:digits50_3}.

Our method advances the OT performance, because OT applications for mass moving assume closeness of target $\mathbb{P}$ and source $\mathbb{Q}$ distributions~\cite{lee2019hierarchical}. While OT maps are $c$-cyclically monotone, i.e., exhibit a specific structure of the map, thus, transportation ${X} \rightarrow {Y}$ via OT maps might not be applied to some problems, see~\cite[Figure 3]{CourtyFTR15} for counter-examples. Moreover, discrete OT techniques are susceptible to regularization terms~\cite{CourtyFTR15, dessein2018regularized, paty2020regularized} and require special scaling~\cite{meng2021large}. 

\begin{table*}[h!]\centering
\begin{center}
\begin{small}
\begin{tabular}{l|l| l l l l  l l}
\hline

Class &Method             & A $\rightarrow$ S    & S$\rightarrow$ A    & A $\rightarrow$ W   & W $\rightarrow$ A             & S $\rightarrow$ W   & W$\rightarrow$ S\\ 

\hline
Deep DA&Source 	&44.5	&6.2	&63.3	&70.3		    &5.5    &45.6 \\ 
        &RSDA-DANN  	&76.1	&80.8	&91.8	&90.5		    &83.1    &70.4 \\
        &SymNet 	&65.9	&86.8	&91.0	&89.2		    &82.2    &56.5 \\
        &CAN   	&\textbf{79.1}	&\textbf{91.2}	&\textbf{92.8}	&\textbf{90.9}		    &\textbf{89.7}    &\textbf{77.9} \\
\hline
Discrete &EMD	 	&38.4	&9.3	&45.2	&45.6	       &13.6	&36.7 \\
OT      &OTLin        	&37.1	&11.0	&38.7	&47.5		&6.2	&39.6 \\
        &Sinkhorn	 	&38.0	&10.1	&44.7	&45.5		&13.1	&37.2 \\
        &SinkhornLpL1 	&38.1	&10.4	&45.2	&45.3		&13.1	&37.2  \\
        &SinkhornL1L2 	&38.1	&10.4	&45.0	&45.3		&13.1	&37.2 \\
\hline
Discrete   &EMD 	&{56.8}	&{29.7}	&{64.9}	&{73.9} &\underline{40.1}	&{60.1} \\
OT(\textbf{Ours)} &OTLin 	&\underline{58.5}	&{29.8}	&\underline{65.2}	&\underline{74.4}		&{40.1}	&\underline{63.1} \\ 
           &Sinkhorn 		&{57.0}	&\underline{31.0}	&{65.2}	&{73.9}		&{39.9}	&{60.0} \\
           &SinkhornLpL1 		&{57.2}	&{31.0}	&{65.2}	&{74.0}		&{40.1}	&{60.1} \\
           &SinkhornL1L2 		&{57.2}	&{31.0}	&{65.2}	&{74.0}	&{40.1}	&{60.1} \\
\hline 

\end{tabular}
\end{small}
\end{center}
\caption{Accuracy$\uparrow$ of domain adaptation in the latent space of ResNet50 model on Modern Office-31 dataset in semi-supervised settings with the 10 known labels for each class in the target domain. By \textbf{bold} we denote highest accuracy, by \underline{underline} we denote highest accuracy for an discrete OT algorithm}
\label{tab:office}
\end{table*}

The method presented in the paper is efficient, because OT can find a map between a target and a source and, at the same time, save discriminativity, i.e., class-wise structure. In our pipeline, OT maps the unlabeled samples to the perturbed samples from the same class in the \textit{source fiction}. In practice, the distance between samples inside the classes is less than the distance between samples from different classes. The Appendix section (\wasyparagraph \ref{sec-details}) presents the adaptation results with only 3 labels per class.
\subsection{Ablation study on $\varepsilon$ parameter}

\begin{figure*}[ht]
 \centering
 \includegraphics[width=\textwidth]{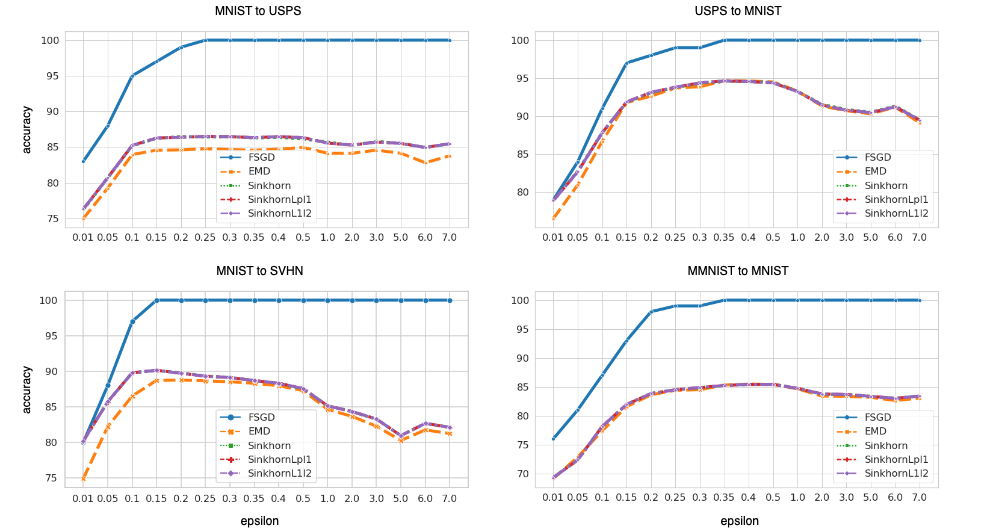}
 \caption{\label{fig:cmon} Results of ablation on $\varepsilon$ parameter for MNIST $\rightarrow$ MNIST-M (left) and SVHN $\rightarrow$ MNIST (right) datasets. FSGD denotes how accurately the source domain model classifies \textit{source fiction} samples obtained with the corresponding parameter $\varepsilon$.}
\end{figure*} 

In this section, we show the adaptation results with different $\varepsilon$ values in the FSGD algorithm.  We evaluated few transportation tasks: MNIST $\rightarrow$ USPS, USPS $\rightarrow$ MNIST, MNIST $\rightarrow$ MNIST-M, and MNIST $\rightarrow$ SVHN datasets with different values of $\varepsilon$. The value of $\frac{1}{2}\min\limits_{n_1,n_2}\|x_{n_1}-x_{n_2}\|$ is different for various  datasets, for SVHN this value is $0.74$, for MNIST is $0.29$, for MNIST-M is $10.9$ and for USPS is equal to $0.85$. We compute these values in the latent space of the ResNet50 classifier trained on the corresponding domain. For each domain, we use different values of $\varepsilon$ and then fit OT to find a map between target and \textit{source fiction}. Our results show a small trade-off between the size of perturbations and the cyclical monotonicity. Figure~\ref{fig:cmon} illustrates the adaptation accuracy with different $\varepsilon$.

When the $\varepsilon$ value becomes larger than $\frac{1}{2}\min$ distance between samples, the accuracy of adaptation decreases because the transformation becomes less $c$-cyclically monotone. With a small value of $\varepsilon$, the adaptation achieves the highest accuracy. With larger perturbations, the prediction of the classifier on perturbed samples becomes more accurate (see FSGD curve in Figure~\ref{fig:cmon}) until perturbation size achieves the $\varepsilon$ bound. With perturbations a little bit larger than the $\varepsilon$ bound, most of the samples in \textit{source fiction} are still $c$-cyclically monotone to the target and the method still works.  

\subsection{Empirical complexity}
\label{sec-complexity}
The method proposed improves discrete OT methods in domain adaptation. This is an important achievement, because discrete OT methods are a very simple and fast technique for DA compared to the approaches based on deep neural networks. Deep domain adaptation methods are complex, typically solve the challenging min-max optimization, and take hours of training. 

In practice, depending on the dataset size and source domain classifier architecture, discrete OT methods like EMD, Sinkhorn, and OTLin take 10-30 minutes to solve domain adaptation on GPU GeForce GTX-1080 (12 GB). As we created the \textit{source fiction} based only on a small number of labeled target samples, an inverse adversarial attack takes less than a minute.

\section{Discussion}
\label{sec-discussion}

\textbf{Potential impact:} The important property of discrete transport methods is their ability to be quickly applied to the new adaptation tasks. Usually, in DA, it is necessary to modify classifier architecture and train separately on each target domain~\cite{GaninL15, cdan, DAN, JAN, mmd}. OT solves domain adaptation by moving target domain samples closer to the source data, and it is unnecessary to change or fine-tune the source classifier. In this paper, we presented an algorithm that improves OT performance. Using OT, a single source domain classifier can make predictions on a range of target domains without any modification.  

Nowadays, the data is shared on separate devices and usually contains personal information, which is inefficient for data transmission and may violate data privacy. The authors of the \cite{SHOT} addresses a challenging DA setting without access to the source data for higher privacy. In our method, we adapt the source classifier to the new domain without access to the source data, using only source classifier, which can be used in various scenarios to avoid privacy issues.

\textbf{Limitations:} The main limitation of our approach is that it is necessary to have access to the labels in the target domain. To avoid the limitation of labels availability , we plan to use pseudo-labeling techniques.

\section{Conclusion}
We proposed an algorithm that modifies DA toward making target data closer to the domain formed by an adversarial attack and proved that adversarial attacks are OT maps over the datasets. We conducted various experiments on different datasets and showed that our method improves OT baselines. In some domains for discrete OT methods, adaptation with \textit{source fiction} improved accuracy by more than $10$ percent. 
  
Our method provides many straightforward applications. Neural OT is applied in the modern generative modeling approach. We plan to adapt our approach to neural transport solvers to make the method robust to out-of-sample estimation. While OT can solve DA with target shift and unbalanced classes~\cite{redko2019optimal, rakotomamonjy2020optimal}, it is promising to use \textit{source fiction} for such problems. Pseudo-labeling techniques to turn the method into a fully unsupervised method is also a promising direction. We expect our research to contribute to the development of less complicated DA techniques and open doors for the future application of $c$-cyclical monotonicity of adversarial attacks.

\textbf{Acknowledgements.} This work was supported by the Analytical Center for the Government of the Russian Federation (IGK 000000D730321P5Q0002), agreement No. 70-2021-00141.

\bibliographystyle{plain}
\bibliography{reference}

\newpage


\appendix

\section{Appendix}
\subsection{Additional background on OT}
OT aims at finding a solution to transfer mass from one distribution to another with the least effort.
Monge's problem was the first example of the OT problem and can be formally expressed as follows: 

\begin{equation}
    \inf _{T \# \mathbb{P}=\mathbb{Q}} \int_{\Omega_{\mathbb{P}}} c(\mathbf{x}, T(\mathbf{x})) \mathbb{P}(\mathbf{x}) d \mathbf{x}
    \label{monge_2}
\end{equation}

The Monge's formulation of OT aims at finding a mapping ${T} : \Omega_{\mathbb{P}} \rightarrow \Omega_{\mathbb{Q}}$ of the two probability measures $\mathbb{P}$ and $\mathbb{Q}$ and a cost function $c: \Omega_{\mathbb{P}} \times \Omega_{\mathbb{Q}} \rightarrow\mathbb{R}_{+}$, where ${T} \# \mathbb{P}_s=\mathbb{Q}_t$ represents the mass preserving push forward operator. 
In Monge's formulation, $T$ cannot split the mass from a single point. 
The problem is that the mapping $T$ may not even exist with such constraints. 

To avoid this, Kantorovitch proposed a relaxation~\cite{villani2008optimal}. 
Instead of obtaining a mapping, the goal is to seek a joint distribution over the source and the target that determines how the mass is allocated. 
For a given cost function $c: \Omega_{\mathbb{P}} \times \Omega_{\mathbb{Q}} \rightarrow\mathbb{R}_{+}$, the primal Kantorovitch formulation can be expressed as the following problem:
\begin{equation}
    \min _{\gamma \in \gamma\left(\mathbb{P}, \mathbb{Q} \right)}\left\{\int_{\Omega_\mathbb{P} \times \Omega_\mathbb{Q}} c(\mathbf{x}, \mathbf{y})d\gamma(\mathbf{x}, \mathbf{y})=\mathbb{E}_{(\mathbf{x}, \mathbf{y}) \sim \gamma}[c(\mathbf{x}, \mathbf{y})]\right\}
\end{equation}

In primal Kantorovitch formulation, we look for a joint distribution $\gamma$ with $\mathbb{P}$ and $\mathbb{P}$ as marginals that minimize the expected transportation cost. 
If the independent distribution $\gamma(\mathbf{x}, \mathbf{y})=\mathbb{P}(\mathbf{x}) \mathbb{Q}(\mathbf{y})$ respects the constraints, linear program is convex and always has a solution for a semi-continuous $c$:
\begin{equation}
    \gamma\left(\mathbb{P}, \mathbb{Q}\right)=\left\{\gamma \in P(\Omega_\mathbb{P}, \Omega_\mathbb{Q}): \int \gamma(\mathbf{x}, \mathbf{y}) d \mathbf{y}=\mathbb{P}(\mathbf{x}), \int \gamma(\mathbf{x}, \mathbf{y}) d \mathbf{x}=\mathbb{Q}(\mathbf{y})\right\}
\end{equation}

The primal Kantorovitch formulation can also be presented in dual form as stated by the Rockafellar---Fenchel theorem~\cite{villani2008optimal}: 
\begin{equation}
    \max _{\phi \in {C}\left(\Omega_{\mathbb{P}}\right), \psi \in {C}\left(\Omega_{\mathbb{Q}}\right)}\left\{\int \phi d \mathbb{P}+\int \psi d \mathbb{Q} \mid \phi(\mathbf{x})+\psi(\mathbf{y}) \leq c(\mathbf{x}, \mathbf{y})\right\}
\end{equation}

After finding a solution to the transport problem, OT measures dissimilarity between the two distributions. 
This similarity is also called the Wasserstein distance~\cite{villani2008optimal}:
\begin{equation}
    W_{p}\left(\mathbb{P}, \mathbb{Q}\right)=\min _{\gamma \in \gamma\left(\mathbb{P}, \mathbb{Q}\right)}\left\{\int_{\Omega_{\mathbb{P}} \times \Omega_{\mathbb{Q}}} c(\mathbf{x}, \mathbf{y}) d \gamma(\mathbf{x}, \mathbf{y})\right\}^{\frac{1}{p}}
\end{equation}
where $c(\mathbf{x}, \mathbf{y})=\|\mathbf{x}-\mathbf{y}\|^{p}$ and $p > 1$. 
The Wasserstein distance encodes the geometry of the space through the optimization problem and can be used on any distribution of mass.

\subsection{Neural Optimal Transport} 

Recently, there has been a solid push to incorporate Input Convex Neural Networks~(ICNNs)~\cite{inputconvex} in OT problems.
According to Rockafellar's Theorem~\cite{rockafellar1966characterization}, every cyclically monotone mapping $g$ is contained in a sub-gradient of some convex function ${f: {X}\rightarrow \mathbb{R}}$. 
Furthermore, according to Brenier's Theorem (Theorem 1.22 of~\cite{Santambrogio}), these gradients uniquely solve the Monge problem~\eqref{monge}.
Following these theorems, a range of approaches explored ICNNs as parameterized convex potentials in dual Kantorovich problem~\cite{Taghvaei, MakkuvaTOL20}. 

Further development of this approach enabled the construction of the non-minimax Wasserstein-2 generative framework~\cite{KorotinW2} that can solve DA and Wasserstein-2 Barycenters estimation~\cite{TaghvaeiBary,KorotinBary}.
Compared to discrete OT, neural methods provide generalized OT methods that can ensure out-of-sample estimates.

\subsection{Cyclical monotonicity bound}
\label{sec-proofs}
\begin{lemma}[cyclical monotonicity of small perturbations of a dataset.]
\label{lemma-main}
Let $x_{1},\dots,x_{N}\in\mathbb{R}^{D}$ be a dataset of $N$ distinct samples. 
Let $x_{1}',\dots,x_{N}'$ be its $\leq \varepsilon$-perturbation, i.e. $\|x_n-x_n'\|\leq \varepsilon$ for all $n=1,2,\dots,N$. Assume that $\varepsilon\leq \frac{1}{2}\min\limits_{n_1,n_2}\|x_{n_1}-x_{n_2}\|$, i.e. the perturbation does not exceed $\frac{1}{2}$ of the minimal pairwise distance between samples. 
Then for all $K$ and $\overline{1,N}$ it holds:
\begin{equation}
\sum_{k=1}^{K}\frac{1}{2}\| x_{n_k}-x_{n_{k}}'\|^{2}\leq\sum_{k=1}^{K}\frac{1}{2}\|x_{n_k}-x_{n_{k+1}}'\|^{2}
\label{cyclical-m-perturb}
\end{equation}
i.e. set $(x_{1},x_{1}'),\dots,(x_{N},x_{N}')$ or, equivalently, the map $x_{k}\mapsto x_{k'}$ is cyclically monotone.
\end{lemma}

\begin{proof}Due to triangle inequality for $\|\cdot\|$, we have
\begin{equation}
    \|x_{n_k}-x_{n_{k+1}}'\|\geq \underbrace{\|x_{n_k}-x_{n_{k+1}}\|}_{\geq 2\varepsilon} - \underbrace{\|x_{n_{k+1}}-x_{n_{k+1}}'\|}_{\leq \varepsilon}=\varepsilon.
    \label{triangle-ineq}
\end{equation}
Taking the square of both sides and summing \eqref{triangle-ineq} for $k=1,2,\dots,K$ yields
\begin{equation}
    \sum_{k=1}^{K}\|x_{n_k}-x_{n_{k+1}}'\|^{2}\geq\sum_{k=1}^{K}\varepsilon^{2}=K\varepsilon^{2}.
    \label{triangle-ineq-sum}
\end{equation}

Due to the assumptions of the lemma, the following inequality holds true:
\begin{equation}\sum_{k=1}^{K}\| x_{n_k}-x_{n_{k}}'\|^{2}\leq\sum_{k=1}^{K}\varepsilon^{2}\leq K\varepsilon^{2}.
\label{upper-bound-ineq}
\end{equation}
We combine \eqref{triangle-ineq-sum} and \eqref{upper-bound-ineq} to obtain
$$\sum_{k=1}^{K}\| x_{n_k}-x_{n_{k}}'\|^{2}\leq\sum_{k=1}^{K}\|x_{n_k}-x_{n_{k+1}}'\|^{2},$$
which is equivalent to
\begin{equation}\sum_{k=1}^{K}c(x_{n_k},x_{n_{k}}')\leq\sum_{k=1}^{K}c(x_{n_k},x_{n_{k+1}}'),
\label{final-ineq}
\end{equation}
and yield $c$-cyclically monotone w.r.t. quadratic cost $c(x,y)=\frac{1}{2}\|x-y\|^{2}$.
\end{proof}

\newpage
\subsection{Experiments and Details}
\label{sec-details}
\textbf{Pre-processing:} We used standard pre-processing over samples with the dataset mean and std normalization. We resize images to 224 x 224 to train the source classifier. No augmentation was used. 

\textbf{Source classifier architectures:} For experiments presented in the main section, we used standard ResNet50~\cite{he2016deep} from \textit{torchvision} framework. We changed the output layer size to be equal to the number of classes in target and source domains. We trained source classifier to achieve ~$90+$ accuracy on the test set of each domain in Digits and Modern Office-31. The classifier was trained using Adam~\cite{Ruder16} optimizer with a $1e-3$ learning rate. The size of latent space before the output layer was equal to $2048$. After training, we applied DA by moving mass in the latent space of the source classifier. The same parameters was used to train the ResNet18 source classifier, this results presented in the next subsection. 

\textbf{Additional experiments: } In this section, we provide additional experiments using the \textit{source fiction} domain for discrete OT solvers. In tables, \textbf{bold} denotes the results of discrete solvers with \textit{source fiction} if this improves its accuracy compared to the standard settings.

Digits dataset results with 3 (Table~\ref{tab:digits50_3}) known labels per class in target domain using ResNet50 classifier.
\begin{table}[h!]\centering
\begin{center}
\begin{footnotesize}
\begin{sc}
\begin{tabular}{l|l|l|l|l|l}
\hline
Method & MNIST    & SVHN    & MNIST   & USPS    & MNIST \\ 
       & SVHN     & MNIST   & USPS    & MNIST   & M-MNIST \\ 

\hline
Source          &22.0   &79.0   &63.0   &80.0   &60.0         \\
EMD	            &21.3	&72.5	&66.1	&67.8	&44.5 \\
Sinkhorn	    &21.7	&73.0	&67.3	&68.7	&44.6 \\
SinkhornLpL1	&21.7	&73.4	&67.3	&68.8	&45.0 \\
SinkhornL1L2	&21.7	&73.4	&67.3	&68.8	&45.0 \\
OTLin	&21.8	&73.4	&67.4	&68.8	&45.0 \\
\hline
EMD	&\textbf{23.1}	&\textbf{83.5}	&\textbf{82.6}	&\textbf{86.5}	&\textbf{54.7} \\
Sinkhorn	&\textbf{23.7}	&\textbf{85.0}	&\textbf{82.6}	&\textbf{86.8}	&\textbf{54.8} \\
SinkhornLpL1	&\textbf{23.7}	&\textbf{85.2}	&\textbf{86.3}	&\textbf{86.9}	&\textbf{54.9} \\
SinkhornL1L2	&\textbf{23.8}	&\textbf{85.2}	&\textbf{86.3}	&\textbf{86.9}	&\textbf{54.9} \\
OTLin			&\textbf{23.9}	&\textbf{85.3}	&\textbf{86.3}	&\textbf{86.9}	&\textbf{55.1} \\		
\hline 

\end{tabular}
\end{sc}
\end{footnotesize}
\end{center}
\caption{Accuracy of domain adaptation by optimal transport in the latent space of ResNet50 model with only 3 known labels for each class in the target domain on Digits datasets. The top part of the table represents semi-supervised settings for discrete OT methods, settings the bottom part presents results using \textit{source fiction}.}
\label{tab:digits50_3}
\end{table}

Modern-Office dataset results with additional domain DLSR (D) (Table~\ref{tab:office_2}), with 10 known labels.
\begin{table}[h!]\centering
\begin{center}
\begin{footnotesize}
\begin{tabular}{l|l l l l l l l l l l l l l}
\hline

Method  & A       & D    & A     & S     & A     & W      & D    & S      & D    &W          & S     & W\\ 
       & D       & A    & S     & A     & W     & A      & S    & D      & W    &D          & W     & S\\ 
\hline
EMD	 &50.7	&46.2	&38.4	&9.3	&45.2	&45.6	&32.7	&16.4	&62.6	&67.1       &13.6	&36.7 \\
OTLin        &45.8	&48.0	&37.1	&11.0	&38.7	&47.5	&36.5	&4.1	&60.9	&61.8	&6.2	&39.6 \\
Sinkhorn	 &51.1	&46.3	&38.0	&10.1	&44.7	&45.5	&32.9	&16.5	&63.5	&67.1	&13.1	&37.2 \\
SinkhornLpL1 &51.1	&46.7	&38.1	&10.4	&45.2	&45.3	&33.0	&16.5	&63.8	&68.3	&13.1	&37.2  \\
SinkhornL1L2 &51.1	&46.7	&38.1	&10.4	&45.0	&45.3	&33.0	&16.5	&63.8	&68.3	&13.1	&37.2 \\
\hline
EMD 	&\textbf{70.9}	&\textbf{72.5}	&\textbf{56.8}	&\textbf{29.7}	&\textbf{64.9}	&\textbf{73.9}	&\textbf{56.6}	&\textbf{47.3}	&\textbf{75.7}	&\textbf{75.1} &\textbf{40.1}	&\textbf{60.1} \\
OTLin 	&\textbf{71.3}	&\textbf{73.6}	&\textbf{58.5}	&\textbf{29.8}	&\textbf{65.2}	&\textbf{74.4}	&\textbf{59.8}	&\textbf{47.3}	&\textbf{76.6}	&\textbf{75.1}	&\textbf{40.1}	&\textbf{63.1} \\ 
Sinkhorn 	&\textbf{70.6}	&\textbf{72.7}	&\textbf{57.0}	&\textbf{31.0}	&\textbf{65.2}	&\textbf{73.9}	&\textbf{56.7}	&\textbf{47.3}	&\textbf{77.4}	&\textbf{74.8}	&\textbf{39.9}	&\textbf{60.0} \\
SinkhornLpL1 	&\textbf{70.6}	&\textbf{72.8}	&\textbf{57.2}	&\textbf{31.0}	&\textbf{65.2}	&\textbf{74.0}	&\textbf{56.8}	&\textbf{47.3}	&\textbf{77.4}	&\textbf{75.5}	&\textbf{40.1}	&\textbf{60.1} \\
SinkhornL1L2 	&\textbf{70.6}	&\textbf{72.8}	&\textbf{57.2}	&\textbf{31.0}	&\textbf{65.2}	&\textbf{74.0}	&\textbf{56.8}	&\textbf{47.3}	&\textbf{77.4}	&\textbf{75.5}	&\textbf{40.1}	&\textbf{60.1} \\
\hline 
\end{tabular}
\end{footnotesize}
\end{center}
\caption{Results of DA in the latent space of ResNet50 model on Modern Office-31 dataset in semi-supervised settings with the the additional DLSA(D) domain. labels are 10 known for each class in the target domain.}
\label{tab:office_2}
\end{table}

Digits(Table~\ref{tab:digits18}) and Modern-Office~\ref{tab:modern18} datasets results using ResNet18 source classifier. 10 labels are known.
\begin{table}[h!]\centering
\begin{center}
\begin{footnotesize}
\begin{sc}
\begin{tabular}{l|l|l|l|l|l}
\hline
Method & MNIST    & SVHN    & MNIST   & USPS    & MNIST \\ 
       & SVHN     & MNIST   & USPS    & MNIST   & M-MNIST \\ 

\hline
Source         &22.0      &79.0       &63.0        &80.0    &60.0         \\
EMD            &15.4      &64.3       &77.0        &80.8    &70.8    \\
Sinkhorn       &16.0      &65.2       &77.9        &81.2    &70.9        \\
Sinkhorn L1Lp  &16.8      &65.7       &79.8        &85.1    &71.8        \\ 
Sinkhorn L1L2  &16.0      &65.7       &79.8        &85.1    &71.8       \\
OTLin           &16.1      &67.1       &79.8         &86.2    &71.8         \\
\hline 

EMD            &\textbf{35.1}     &\textbf{87.1}       &\textbf{85.2}      &\textbf{95.2}     &\textbf{82.3}    \\
Sinkhorn       &\textbf{38.0}     &\textbf{88.3}       &\textbf{88.3}      &\textbf{95.2}     &\textbf{83.5}        \\
Sinkhorn L1Lp  &\textbf{37.1}     &\textbf{88.3}       &\textbf{88.2}      &\textbf{95.2}     &\textbf{83.5}       \\ 
Sinkhorn L1L2  &\textbf{37.1}     &\textbf{88.3}       &\textbf{88.3}      &\textbf{95.2}     &\textbf{83.5}       \\
OTLin           &\textbf{38.0}     &\textbf{90.0}       &\textbf{88.2}      &\textbf{95.2}     &\textbf{83.5}        \\

\hline 

\end{tabular}
\end{sc}
\end{footnotesize}
\end{center}
\caption{Accuracy of domain adaptation by optimal transport in the latent space of ResNet18 model with the 10 known labels for each class in the target domain on Digits datasets. The top part of the table represents semi-supervised settings for discrete OT methods, the bottom part presents results using \textit{source fiction}.}
\label{tab:digits18}
\end{table}

\begin{table}[h!]\centering
\begin{center}
\begin{footnotesize}
\begin{sc}
\begin{tabular}{l|l|l|l|l|l|l|l|l|l|l|l}
\hline

Method & A    & D   & A   & S   & A  & W    & D    & S   & D    & S  & W\\ 
       & D    & A   & S    & A   & W  & A    & S    & D   & W    & W  & S\\

\hline
Source	        &37	&43	&25	&8	&57	&41	&20	&2	&65	&3	&27\\
EMD	            &57	&42	&27	&16	&53	&38	&32	&28	&53	&19	&37\\
Sinkhorn	    &58	&43	&27	&17	&53	&40	&32	&29	&54	&20	&38\\
SinkhornLpL1	&58	&43	&27	&17	&53	&40	&32	&29	&53	&20	&38\\
SinkhornL1L2	&58	&43	&27	&17	&53	&40	&32	&29	&53	&20	&38\\
OTLin	        &58	&43	&27	&17	&53	&40	&32	&29	&53	&20	&38\\
\hline
EMD	            &\textbf{76}	&\textbf{67}	&\textbf{51}	&\textbf{44}	&\textbf{73}	&\textbf{63}	&\textbf{59}	&\textbf{60}	&\textbf{76}	&\textbf{45}	&\textbf{61}\\
Sinkhorn	    &\textbf{77}	&\textbf{67}	&\textbf{51}	&\textbf{43}	&\textbf{75}	&\textbf{63}	&\textbf{59}	&\textbf{57}	&\textbf{77}	&\textbf{41}	&\textbf{60}\\
SinkhornLpL1	&\textbf{77}	&\textbf{67}	&\textbf{51}	&\textbf{43}	&\textbf{75}	&\textbf{63}	&\textbf{59}	&\textbf{57}	&\textbf{77}	&\textbf{41}	&\textbf{60}\\
SinkhornL1L2	&\textbf{77}	&\textbf{67}	&\textbf{51}	&\textbf{43}	&\textbf{75}	&\textbf{63}	&\textbf{59}	&\textbf{57}	&\textbf{77}	&\textbf{41}	&\textbf{60}\\
OTLin	        &\textbf{77}	&\textbf{67}	&\textbf{51}	&\textbf{43}	&\textbf{75}	&\textbf{63}	&\textbf{59}	&\textbf{57}	&\textbf{77}	&\textbf{41}	&\textbf{60}\\ \hline

\end{tabular}
\end{sc}
\end{footnotesize}
\end{center}
\caption{Results of domain adaptation in the latent space of ResNet18 model on Modern Office-31 dataset in semi-supervised settings with the 10 known labels for each class in the target domain}
\label{tab:modern18}
\end{table}

CIFAR10-STL10 adaptation task results (Table~\ref{table:cifarstl}) with ResNet18 source classifier.
\begin{table}[h!]\centering
\begin{center}
\begin{footnotesize}
\begin{sc}
\begin{tabular}{llllllll}
\hline
Method & Source & EMD & Sinkh & Sinkh L1Lp & Sinkh L1L2 &  OTLin \\ 
\hline
CIFAR$\rightarrow$STL           &37.0      &48.1                &48.1                         & 48.0                    & 48.0                  & 48.1          \\
CIFAR$\rightarrow$SF            &         &\textbf{51.1}       &\textbf{51.0}                 & \textbf{51.0}          & \textbf{51.0}          & \textbf{51.0}        \\
\hline
STL$\rightarrow$CIFAR           &75.0     &74.1                 &74.1                     & 74.1                   & 74.1                        & 74.1        \\
STL$\rightarrow$SF              &         &\textbf{76.2}       &\textbf{76.2}             &\textbf{76.2}           & \textbf{76.2}              & \textbf{76.2}        \\ 
\hline \\

\end{tabular}
\end{sc}
\end{footnotesize}
\end{center}
\caption{Results on CIFAR-10 and STL dataset in semi-supervised settings. SF is source fiction. The top table presented results for the settings with the 10 known labels for each class in the target domain, and the bottom table presents the result with the 100 known labels for each class}.
\label{table:cifarstl}
\end{table}

\end{document}